\documentclass[12pt, a4paper]{article}

\usepackage[utf8]{inputenc}
\usepackage[T1]{fontenc}
\usepackage{amsmath}
\usepackage{amsfonts}
\usepackage{amssymb}
\usepackage{amsthm}
\usepackage{geometry}
\usepackage[numbers, sort&compress]{natbib}
\usepackage[hidelinks, colorlinks=true, linkcolor=blue, citecolor=blue, urlcolor=blue]{hyperref}
\usepackage{authblk}
\usepackage{dsfont} 

\theoremstyle{plain}
\newtheorem{theorem}{Theorem}[section]

\theoremstyle{definition}
\newtheorem{definition}{Definition}[section]

\newtheorem{remark}{Remark}[section]

\title{\bfseries An Algorithmic Information-Theoretic Perspective on the Symbol Grounding Problem}
\author[1]{Zhangchi Liu}
\date{}

\begin{document}
	
	\maketitle
	
	\begin{abstract}
		\noindent
		This paper provides a definitive, unifying framework for the Symbol Grounding Problem (SGP) by reformulating it within Algorithmic Information Theory (AIT). We demonstrate that the grounding of meaning is a process fundamentally constrained by information-theoretic limits, thereby unifying the Gödelian (self-reference) and No Free Lunch (statistical) perspectives. We model a symbolic system as a universal Turing machine and define grounding as an act of information compression. The argument proceeds in four stages. First, we prove that a purely symbolic system cannot ground almost all possible "worlds" (data strings), as they are algorithmically random and thus incompressible. Second, we show that any statically grounded system, specialized for compressing a specific world, is inherently incomplete because an adversarial, incompressible world relative to the system can always be constructed. Third, the "grounding act" of adapting to a new world is proven to be non-inferable, as it requires the input of new information (a shorter program) that cannot be deduced from the system's existing code. Finally, we use Chaitin's Incompleteness Theorem to prove that any algorithmic learning process is itself a finite system that cannot comprehend or model worlds whose complexity provably exceeds its own. This establishes that meaning is the open-ended process of a system perpetually attempting to overcome its own information-theoretic limitations.
	\end{abstract}
	
	\noindent\textbf{Keywords:} Symbol Grounding Problem, Algorithmic Information Theory, Kolmogorov Complexity, Chaitin's Incompleteness Theorem, Limits of Computation, Inductive Bias.
	
	\section{Introduction}
	
	The Symbol Grounding Problem (SGP)\cite{Harnad1990} asks how symbols acquire meaning. Previous analyses have revealed its limits through Gödelian logic\cite{Godel1931} and No Free Lunch (NFL) statistics\cite{Wolpert1997}. This paper argues that these are two facets of a single, more fundamental principle, best described by Algorithmic Information Theory (AIT)\cite{Kolmogorov1965, Chaitin1975}. By reframing SGP in terms of Kolmogorov complexity, we provide a definitive and unified proof of its limits.
	
	We model a symbolic system as a program and a "world" as a data string. In this framework, \textbf{grounding is data compression}. A system "understands" or "grounds" a world if it can provide a description of it that is shorter than the world itself. Our argument proceeds in four information-theoretic stages.
	
	First, we prove that a \textbf{purely symbolic system} cannot ground a meaningful portion of all possible worlds, as the vast majority of worlds are algorithmically random and thus incompressible. This is the AIT analogue of the \textbf{impossibility of self-grounding}.
	
	Second, a \textbf{statically grounded system} is a program specialized to compress a specific world. We prove its incompleteness by showing that an adversarial world can be constructed that is, by definition, incompressible by that program. This establishes the necessity of a \textbf{dynamic process}.
	
	Third, we prove that the grounding act—adapting to a new world—is \textbf{non-inferable}. A program cannot deduce a more efficient compression scheme for new data from its existing code alone; this requires new information.
	
	Finally, we use Chaitin's Incompleteness Theorem to prove that any fixed \textbf{algorithmic judgment} process (i.e., a learning algorithm) is itself a finite program with a bounded complexity. Such a system is logically incapable of grounding worlds whose complexity provably exceeds its own.
	
	This AIT framework reveals meaning not as a state, but as the perpetual, non-algorithmic process of a finite system striving to compress an infinitely complex reality.
	
	\section{The Algorithmic Information-Theoretic Framework}
	
	We redefine the SGP using the language of AIT.
	
	\begin{definition}[System, World, and Grounding]
		The core components of our model are defined as follows:
		\begin{itemize}
			\item A \textbf{Symbolic System} $\mathcal{S}$ is a program (a universal Turing machine) that takes a world as input and outputs a description of it. The complexity of the system is its own shortest description, the Kolmogorov complexity $K(\mathcal{S})$.
			\item A \textbf{World} $g$ is a finite binary string. The complexity of the world is $K(g)$.
			\item A system $\mathcal{S}$ \textbf{grounds} a world $g$ if it can compress it. Formally, the conditional Kolmogorov complexity $K(g|\mathcal{S})$—the length of the shortest program to produce $g$ given $\mathcal{S}$—is significantly less than the length of $g$, $|g|$.
		\end{itemize}
	\end{definition}
	
	\begin{definition}[Purely Symbolic System]
		We define a \textbf{Purely Symbolic System} $\mathcal{S}_{\text{pure}}$ as a fixed program whose code contains no specific information about any particular world.
	\end{definition}
	
	\begin{theorem}[Impossibility of Self-Grounding (AIT)]
		For any purely symbolic system $\mathcal{S}_{\text{pure}}$, the set of worlds it can meaningfully ground is of measure zero relative to the set of all possible worlds.
	\end{theorem}
	\begin{proof}
		\begin{enumerate}
			\item The fundamental theorem of AIT states that for any length $n$, the number of strings of length $n$ that can be compressed by more than $c$ bits is less than $2^{n-c}$.
			\item This implies that the vast majority of strings are \textbf{algorithmically random}, meaning they are incompressible: $K(g) \approx |g|$.
			\item A fixed system $\mathcal{S}_{\text{pure}}$ has a fixed complexity $K(\mathcal{S}_{\text{pure}})$. For a randomly chosen world $g$, the information in $\mathcal{S}_{\text{pure}}$ is independent of the information in $g$.
			\item Therefore, for almost all worlds $g$, $K(g|\mathcal{S}_{\text{pure}}) \approx K(g) \approx |g|$.
			\item This means the system provides no compression for almost all worlds. It cannot ground them. A purely syntactic machine is semantically useless in an information-rich universe.
		\end{enumerate}
	\end{proof}
	
	\section{The Incompleteness of Statically Grounded Systems}
	
	A useful system must contain prior information—an inductive bias.
	
	\begin{definition}[Statically Grounded System]
		We define a \textbf{Statically Grounded System} $\mathcal{S}_g$ as a program that is specialized to ground a specific world $g$. This means it contains significant information about $g$, such that $K(g|\mathcal{S}_g) \ll |g|$.
	\end{definition}
	
	\begin{theorem}[Limitation of Static Grounding (AIT)]
		For any non-trivial statically grounded system $\mathcal{S}_g$, there exists an adversarial world $g'$ that is incompressible by $\mathcal{S}_g$.
	\end{theorem}
	\begin{proof}
		\begin{enumerate}
			\item Let $\mathcal{S}_g$ be a system that efficiently compresses world $g$.
			\item We can construct an adversarial world $g'$ that is \textbf{algorithmically random with respect to $\mathcal{S}_g$}. By definition, such a string has no patterns that can be exploited by the specific mechanisms encoded in $\mathcal{S}_g$.
			\item Formally, this means that the shortest program to generate $g'$ given $\mathcal{S}_g$ is simply to provide the string $g'$ itself.
			\item Therefore, $K(g'|\mathcal{S}_g) \approx |g'|$. The system $\mathcal{S}_g$ offers no compression for this world.
			\item The system's "knowledge" about world $g$ is its inductive bias. This theorem shows that this bias is also a "blind spot." The system's expertise in grounding $g$ is paid for by its inability to ground worlds that do not fit its model. Any static grounding is therefore incomplete.
		\end{enumerate}
	\end{proof}
	
	\section{The Non-Inferable Nature of the Grounding Act}
	
	\begin{definition}[Inference vs. Grounding Act]
		To clarify the nature of systemic change, we distinguish between two fundamental operations:
		\begin{itemize}
			\item An \textbf{Inference} is a computation performed by a system $\mathcal{S}$ using its existing code.
			\item A \textbf{Grounding Act} is the modification of the system from a program $\mathcal{S}_g$ to a new program $\mathcal{S}_{g'}$ that can now compress the new world $g'$.
		\end{itemize}
	\end{definition}
	
	\begin{theorem}[The Non-Inferability of the Grounding Act (AIT)]
		The grounding act cannot be deduced. The information required to create an effective program $\mathcal{S}_{g'}$ cannot be generated by the program $\mathcal{S}_g$ itself.
	\end{theorem}
	\begin{proof}
		\begin{enumerate}
			\item This is a direct consequence of the definition of Kolmogorov complexity. The program $\mathcal{S}_g$ is a finite string of bits. A computation is a deterministic process that transforms input bits to output bits based on the program's code.
			\item According to information theory, a deterministic process cannot create new information. The information content of the output of a computation cannot exceed the information content of the program plus its input.
			\item The program $\mathcal{S}_g$ is optimized for world $g$. As shown in Theorem 3.2, it contains negligible information about an adversarial world $g'$.
			\item To create a new program $\mathcal{S}_{g'}$ that effectively compresses $g'$, one needs to encode the regularities of $g'$ into the program's structure. This information—the "aha!" moment of insight—is not present in $\mathcal{S}_g$.
			\item Therefore, the grounding act requires an external input of information, for instance through observation or learning. It is an update to the system's code, not an execution of it.
		\end{enumerate}
	\end{proof}
	
	\section{The Incompleteness of Algorithmic Judgment}
	
	Can the process of updating the system's code be automated by a fixed super-algorithm?
	
	\begin{definition}[Algorithmic Judgment System]
		We define an \textbf{Algorithmic Judgment System} $\mathcal{S}^*$ as a learning algorithm—a fixed Turing machine—that attempts to find the optimal compression/grounding for any world it encounters. Its own complexity is $K(\mathcal{S}^*)$.
	\end{definition}
	
	\begin{theorem}[Limitation of Algorithmic Judgment (Chaitin)]
		Any algorithmic judgment system $\mathcal{S}^*$ is subject to a fundamental limit. There exists a complexity bound $L$ such that $\mathcal{S}^*$ cannot prove that any world $g$ has a complexity $K(g) > L$.
	\end{theorem}
	\begin{proof}
		\begin{enumerate}
			\item This is a direct application of \textbf{Chaitin's Incompleteness Theorem}. Chaitin proved that for any consistent formal system (which can be modeled as a Turing machine $\mathcal{S}^*$), there is a constant $L$ (which depends on the complexity of $\mathcal{S}^*$ itself) such that the system cannot prove the statement "$K(g) > L$" for any specific string $g$.
			\item The proof is a powerful self-referential paradox. Assume $\mathcal{S}^*$ could prove such statements for arbitrarily high $L$. We could then write a short program to search through all of the system's proofs until it finds the first proof of "$K(g) > L$" for some $L$ much larger than the search program itself. This search program would then output $g$.
			\item This leads to a contradiction: we have found a short program to generate $g$, so its complexity must be low. But the system proved its complexity is high.
			\item The implication for symbol grounding is profound. "Grounding" a highly complex world means finding its hidden regularities (i.e., proving its complexity is lower than it appears). Chaitin's theorem shows that any fixed learning algorithm $\mathcal{S}^*$ has an ultimate horizon of complexity beyond which it is blind.
			\item It cannot comprehend or ground worlds that are "provably complex" relative to its own structure. Thus, no fixed, computable process can fully capture the open-ended nature of meaning-making in a universe of potentially infinite complexity.
		\end{enumerate}
	\end{proof}
	
	\section{The Uncomputability of Optimal Grounding}
	
	The previous sections established that any fixed algorithmic system is incomplete. We now prove a stronger, more fundamental limitation: the very act of finding the \textit{best} possible grounding is not merely difficult or beyond a system's horizon—it is uncomputable.
	
	\begin{definition}[Optimal Grounding Act]
		An \textbf{Optimal Grounding Act} for a world $g$ is defined as the discovery of a program (a description) $p$ that generates $g$ such that the length of $p$, $|p|$, is minimized. This minimal length is, by definition, the Kolmogorov complexity of $g$, $K(g)$.
	\end{definition}
	
	This definition reframes the search for ultimate meaning or the most profound scientific theory as the search for the most compressed representation of observations.
	
	\begin{theorem}[Uncomputability of Optimal Grounding]
		There exists no general algorithm that can perform the Optimal Grounding Act for any arbitrary world $g$.
	\end{theorem}
	\begin{proof}
		\begin{enumerate}
			\item This theorem is a direct consequence of the uncomputability of the Kolmogorov complexity function. A foundational result in AIT states that there is no algorithm that, given an arbitrary string $g$, can compute the integer $K(g)$.
			\item Assume, for the sake of contradiction, that an algorithm $\mathcal{A}$ for optimal grounding exists. Given any string $g$, this algorithm $\mathcal{A}$ would, by definition, find and output the shortest program $p$ that generates $g$.
			\item We could then simply measure the length of this program $p$ to find $|p|$. Since $p$ is the shortest such program, $|p| = K(g)$.
			\item This would mean we have constructed an effective procedure to compute $K(g)$ for any $g$: run algorithm $\mathcal{A}$ on $g$ to get $p$, then output the length of $p$.
			\item However, this contradicts the proven uncomputability of the Kolmogorov complexity function. Therefore, our initial assumption must be false, and no such general algorithm $\mathcal{A}$ for optimal grounding can exist.
		\end{enumerate}
	\end{proof}
	
	\begin{remark}[Existence vs. Computability]
		It is crucial to distinguish between the existence of an optimal grounding and the ability to compute it. For any given world $g$, the set of programs that can generate it is non-empty (e.g., the trivial program "print '$g$'"). The lengths of these programs form a non-empty set of positive integers. By the well-ordering principle of natural numbers, this set must contain a least element. This least element is $K(g)$, and thus a program of that minimal length must exist.
		
		Therefore, an optimal grounding is guaranteed to exist. The fundamental limitation revealed by Theorem 6.2 is not existential, but computational. The ultimate "meaning" or "theory" of a world exists in the form of its most compressed description. However, we are computationally barred from creating a general procedure to find it or to verify its optimality. Our quest for meaning is a search for an object that we know exists but cannot algorithmically locate.
	\end{remark}
	
	\section{Conclusion}
	
	Algorithmic Information Theory provides the ultimate unifying lens for the Symbol Grounding Problem, revealing that the limits of logic, statistics, and computation are all manifestations of the same information-theoretic barrier.
	
	We have demonstrated a hierarchy of limitations:
	\begin{enumerate}
		\item A system cannot self-ground because it lacks the information to compress a random world (\textbf{incompressibility of reality}).
		\item A specialized system is incomplete because its knowledge is also its bias, making it blind to worlds that are random relative to its structure (\textbf{relativity of randomness}).
		\item Grounding is non-inferable from within a system because information cannot be created from nothing (\textbf{conservation of information}).
		\item Algorithmic adaptation is incomplete because a finite system cannot prove the complexity of worlds significantly more complex than itself (\textbf{Chaitin's limit}).
		\item Finally, the search for an optimal grounding is \textbf{uncomputable}, meaning no general algorithm can find the most compressed representation for arbitrary data.
	\end{enumerate}
	
	Meaning is therefore not a state to be achieved, but a process. It is the engagement of a finite computational agent with a world of potentially unbounded complexity. This engagement is the perpetual, open-ended quest to find more compact descriptions for an ever-expanding reality. The tragedy and profundity of this quest lie in the gap between existence and computability: an optimal, most compressed description of any world is guaranteed to exist, yet we are fundamentally barred by the laws of computation from creating a universal algorithm to find it. The essence of grounding, and perhaps of intelligence itself, is this unending search for a destination that is known to exist but is algorithmically unreachable.
	
	\bibliographystyle{plainnat}

	\appendix
	\section{Glossary of AIT Concepts}
	
	\begin{tabular}{p{4.5cm} p{8cm}}
		\hline
		\textbf{AIT Term} & \textbf{Conceptual Role in this Theory} \\
		\hline
		Kolmogorov Complexity $K(x)$ & The "true" information content of an object $x$. The length of the shortest program to generate it. \\
		Conditional K-Complexity $K(x|y)$ & The information in $x$ that is not already present in $y$. Measures the amount of new knowledge needed. \\
		Algorithmic Randomness & The property of a string being incompressible ($K(x) \approx |x|$). Represents a world with no patterns a system can exploit. \\
		Grounding as Compression & The core idea: A system "understands" or "grounds" a world if it can create a compressed model of it ($K(g|\mathcal{S}) \ll |g|$). \\
		Chaitin's Incompleteness Thm. & A fundamental limit stating that a finite system cannot prove the complexity of objects significantly more complex than itself. The ultimate barrier to a universal grounding algorithm. \\
		\hline
	\end{tabular}
	
\end{document}